\documentclass[10pt,twocolumn,letterpaper]{article}

\usepackage[pagenumbers]{cvpr}

\usepackage{amsmath,amsfonts,bm}

\def\eqref#1{equation~\ref{#1}}

\def\1{\bm{1}}

\DeclareMathAlphabet{\mathsfit}{\encodingdefault}{\sfdefault}{m}{sl}
\SetMathAlphabet{\mathsfit}{bold}{\encodingdefault}{\sfdefault}{bx}{n}

\usepackage{graphicx}
\usepackage{amsmath}
\usepackage{amssymb}
\usepackage{amsthm}
\usepackage{booktabs}
\usepackage{multirow}
\usepackage{makecell}
\usepackage{adjustbox}
\newtheorem{prop}{Proposition}

\usepackage[pagebackref,breaklinks,colorlinks]{hyperref}

\usepackage[capitalize]{cleveref}
\crefname{section}{Sec.}{Secs.}
\Crefname{section}{Section}{Sections}
\Crefname{table}{Table}{Tables}
\crefname{table}{Tab.}{Tabs.}

\def\confName{CVPR}
\def\confYear{2022}

\begin{document}

\title{Towards Tokenized Human Dynamics Representation}

\author{%
    Kenneth Li$^{1,2}$ \quad Xiao Sun$^1$ \quad Zhirong Wu$^1$ \quad Fangyun Wei$^1$ \quad Stephen Lin$^1$ \\
    $^1$Microsoft Research Asia \quad $^2$Harvard University\\
    \texttt{\{xias, wuzhiron, fawe, stevelin\}@microsoft.com} \quad \texttt{ke\_li@g.harvard.edu} \\
}

\twocolumn[{%
	\maketitle
	\vspace{-0.75cm}
	\renewcommand\twocolumn[1][]{#1}%
	\begin{center}
		\centering
        \includegraphics[width=\linewidth]{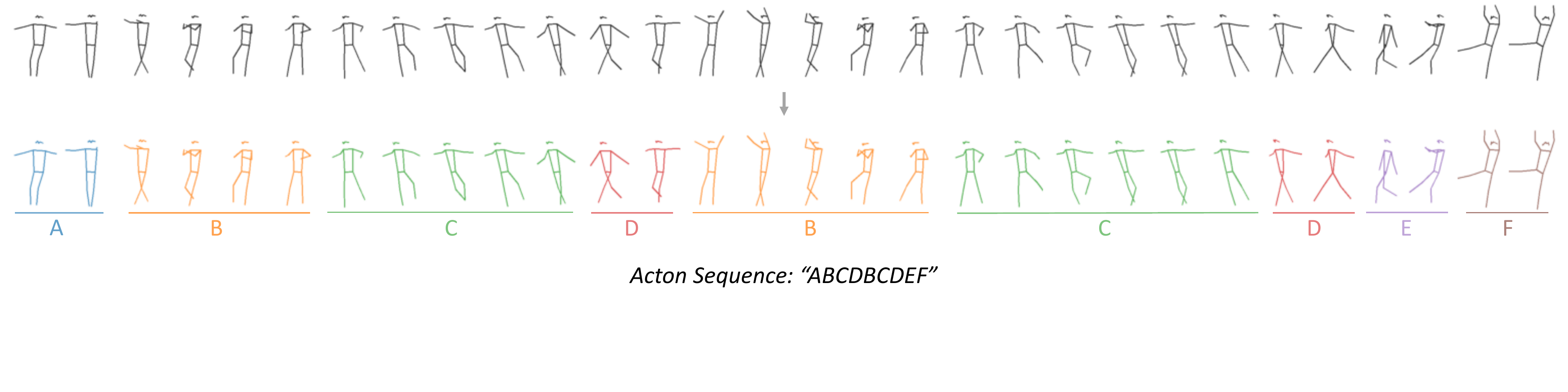}
        \vspace{-18mm}
        \captionof{figure}{We present a self-supervised technique for discovering recurring temporal patterns, called {\em actons}, in long kinematic videos like human dance. From a collection of such videos without any annotations, we extract a set of actons and use this lexicon to segment and model motion sequences as shown above (skeleton sequence temporally downsampled by 10$\times$).}
   \label{teaser}
	\end{center}%
}]

\begin{abstract}
\vspace{-5mm}
For human action understanding, a popular research direction is to analyze short video clips with unambiguous semantic content, such as jumping and drinking. However, methods for understanding short semantic actions cannot be directly translated to long human dynamics such as dancing, where it becomes challenging even to label the human movements semantically.
Meanwhile, the natural language processing (NLP) community has made progress in solving a similar challenge of annotation scarcity by large-scale pre-training, which improves several downstream tasks with one model. In this work, we study how to segment and cluster videos into recurring temporal patterns in a self-supervised way, namely acton discovery, the main roadblock towards video tokenization. We propose a two-stage framework that first obtains a frame-wise representation by contrasting two augmented views of video frames conditioned on their temporal context. The frame-wise representations across a collection of videos are then clustered by K-means. Actons are then automatically extracted by forming a continuous motion sequence from frames within the same cluster. We evaluate the frame-wise representation learning step by Kendall's Tau and the lexicon building step by normalized mutual information and language entropy. We also study three applications of this tokenization: genre classification, action segmentation, and action composition. On the AIST++ and PKU-MMD datasets, actons bring significant performance improvements compared to several baselines.\footnote{Accomplished during K. Li's internship at Microsoft Research Asia.} \footnote{Code: \url{https://github.com/likenneth/acton}.}
\vspace{-4mm}
\end{abstract}

\section{Introduction}
In the past decade, computer vision research owes much of its success to the construction of large labeled datasets.
The labels in the datasets provide semantics that associate visual data with natural language descriptions. Labeling human actions by semantic descriptions, however, limits the scope of the action space, as not all human actions can be clearly and unambiguously defined by language. Because of this, different datasets ask for a different facet of understanding for video, leading to highly specialized models.

Facing a similar challenge, the natural language processing (NLP) community proposes to use a large-scale pre-trained language model~\cite{vaswani2017attention,brown2020language} to holistically understand language and improve multiple downstream tasks with one model, such as machine translation, reading comprehension, and emotion classification. In order to bring these advantages into human dynamics analysis, we need to first tokenize videos. Similar to moving from speech-based language processing to (sub)word-based processing, tokzenization would not only enable processing longer videos, but also make the prediction for downstream tasks more transparent. 

Long and complex human actions in videos are difficult to depict by semantic language descriptions.
Though there exist a few iconic movements that are associated with a clear label, such as ``the moonwalk'' by Michael Jackson, a majority of dance movements are best conveyed by showing examples.
Due to this complexity, a written transcription of dance uses a figure-based notation system called Labanotation~\cite{hutchinson1977labanotation} to record human motions.

In this work, we raise the question of how to learn meaningful tokenzied representations of a long human dance without any supervision.
We observe that, though complex over a long time period, human dances often exhibit recurring temporal patterns.
We thus propose an atomic representation called {\em actons} as a mid-level representation for modeling long human action videos.
Given this mid-level representation, human dances could be represented by a sequence of acton tokens, similar to words of a lexicon in a sentence.
Building this acton dictionary may facilitate further applications such as action retrieval, dance translation, and motion stylization. 

We propose a two-step self-supervised approach for clustering and segmenting long human dances into actons.
First, we obtain a frame-level representation by training a Temporal Alignment Network (TAN) through contrastive learning~\cite{chen2020simple}.
Given a 3D skeleton sequence, we augment the sequence into two views by considering rotation, translation and speed augmentations.
Frames in one view are mapped to corresponding frames in the other view as positives.
Contrastive learning then discriminates between positive and negative pairs at the frame level using a Transformer backbone. TAN is discussed in Section~\ref{sec: tan}.

With the pretrained TAN, the embeddings can be used to retrieve frames in a similar action context.
We therefore propose a simple way to discover actons, by jointly clustering all frames in a video dataset by K-means.
Actons are then automatically segmented from the long actions by finding continuous sequences within the same cluster assignment. Unsupervised lexicon building is discussed in Section~\ref{sec: cluster}.

We conduct experiments on the AIST++ and PKU-MMD datasets~\cite{aist-dance-db,li2021learn,liu2017pku}, which provide 3D skeletons from multi-view reconstruction. The AIST++ dataset contains dances of 10 genres with basic motions and advanced complex motions; PKU-MMD is an action detection dataset containing 1076 long video sequences. Alignment quality is measured by Kendall's Tau~\cite{kendall1938new, dwibedi2019temporal} and clustering performance is evaluated by normalized mutual information and language entropy.
We also present three applications of the acton representation: genre classification, action detection, and action composition.
Extensive experimental results demonstrate the effectiveness of the method against established baselines, TCN and TCC~\cite{dwibedi2019temporal,sermanet2018time}. We summarize the contributions of this work as follows:
\begin{itemize}
    \item This work raises a novel question of how to discover meaningful elementary representations from a long human action sequence \textbf{without any supervision}, namely, unsupervised acton discovery. Analogous to words in a language, the discovered actons can serve as words in a body language, which then can be further used for action assessment, translation and composition. This is valuable to domain specialists like choreographers, athletes and so on. To our knowledge, this is the first time that this interesting, important and challenging task has been presented, discussed, explored and evaluated in depth.
    \item We believe the biggest technical challenges of this task lie in learning strong frame-wise, context-aware and temporally aligned features. To this end, we propose a novel Temporal Alignment Network (TAN) together with effective speed augmentation and negative sampling techniques to address the above issues.   
    \item Extensive experiments, including ablation studies on the augmentation and negative sampling strategies, as well as comparison with the state-of-the-art self-supervised sequence learning methods Time-Contrastive Networks (TCN)~\cite{dwibedi2019temporal} and Temporal Cycle-Consistency (TCC)~\cite{sermanet2018time}, demonstrate the effectiveness of the proposed method.
    \item We use the discovered actons as tokenized features and apply them to the application of genre classification, action detection and choreography based on random acton composition.
\end{itemize}

\section{Related Works}

\paragraph{Self-Supervised Learning for Action Recognition}  
An effective representation for action recognition should capture the temporal order of human movement, the speed of movement, as well as the detailed motions.
For self-supervised learning, pretext tasks which encode such representations include learning the arrow of time~\cite{wei2018learning}, shuffle and learn~\cite{misra2016shuffle,lee2017unsupervised}, and learning the temporal speed~\cite{benaim2020speednet}.
Recent advances on contrastive learning for image representation learning~\cite{wu2018unsupervised,chen2020simple} also show promising results on videos.
Research along this direction has achieved the state-of-the-art by contrasting on the temporal dimension~\cite{gordon2020watching,sermanet2018time}, distilling motion representations~\cite{han2020self}, and designing better video augmentations~\cite{qian2020spatiotemporal}.

The Kinetics human action video dataset~\cite{kay2017kinetics} is the most popular dataset for pretraining action representations. However, the dataset is curated since each video is trimmed to temporally focus on the underlying action.
The holistic representations for action recognition learned from Kinetics cannot be scaled to long kinematic videos that describe a series of complex motions.

\paragraph{Video Alignment} 
In contrast to holistic representations, a frame-level action representation may effectively capture temporal action progressions. For example, the act of pouring may consist of sub-action phases such as grabbing a bottle, tilting it to allow liquid to flow out, and putting down the bottle. The transformation which maps the frames of one video to another can be learned by CCA~\cite{andrew2013deep} and Gaussian mixture~\cite{sener2018unsupervised,swetha2021unsupervised} losses.
Self-supervised approaches to alignment include cycle consistency~\cite{dwibedi2019temporal} and time-contrastive networks~\cite{sermanet2018time}.
The video alignment task assumes that all data in the same action category are loosely aligned.
Our task is much more challenging because the training data is uncurated, spanning multiple categories of actions.

\paragraph{Unsupervised Pattern Discovery from Sequences} 
Natural sensory signals, of audio, vision or text, often come in the form of long sequences.
An important research topic~\cite{oates2002peruse} is to discover recurring patterns in the sequences.
Early works~\cite{de1996unsupervised,park2007unsupervised} show that it is possible to extract words and linguistic entities from audio data without any supervision.
Further research~\cite{myers1981level} uses recurring words for connected word recognition. 
Recently, a zero resource speech challenge~\cite{nguyen2020zero} has demonstrated that neural language models can be successfully learned from speech recordings. 
Pattern discovery models have also been applied to discover meaningful patterns in music~\cite{dannenberg2003pattern}.
Research on bioinformatics~\cite{brazma1998approaches,rigoutsos1998combinatorial} has found structurally important gene patterns (motifs) in DNA sequences. The key method for pattern discovery is to perform sequence-to-sequence alignment, e.g. using dynamic time warping~\cite{sakoe1978dynamic}, and aggregate aligned sequences into clusters. 

For applications of unsupervised pattern discovery on videos, topic models such as latent Dirichlet allocation are applied to surveillance videos to discover recurrent activities on crossroads~\cite{emonet2011extracting,hospedales2009markov,emonet2013temporal,varadarajan2013sequential}. The alignment between short video clips is measured by the similarity of features, which often use histograms of optical flow vectors. Instead of traffic activities, the goal of our work is to analyze human motions, which are of significantly greater complexity.

\paragraph{Atomic Actions} 
In contrast to action classification, the study of atomic actions aims to provide a detailed representation of complex action sequences. With its combinatorial structure, such a representation can also be used to reduce the complexity of recognition systems. In pursuit of this, 
FineGym~\cite{shao2020finegym} provides a fine-grained temporal annotation for gymnastic actions with up to 530 acton classes, which are organized into three semantic and two temporal hierarchies. Despite the expense to annotate them, acton classes are generally not shared between activities, e.g., between cooking and gymnastics. Therefore, a method for simultaneous self-supervised acton discovery and segmentation is needed to facilitate analysis of complex human actions.
Methods have been proposed for unsupervised segmentation of complex activities~\cite{sener2018unsupervised,swetha2021unsupervised}, but they presume knowledge of acton classes and canonical orders. VideoBERT\cite{sun2019videobert} uses clustered visual words but does not take different action speed into consideration by assuming an identical word length.

\section{Methodology}
We start by analyzing the technical goal of acton discovery and proposing a two-stage framework. Then we elaborate on a design choice for each of the two stages. 
\subsection{Two-stage Framework}
\begin{figure*}[t]
\centering
\includegraphics[width=\linewidth]{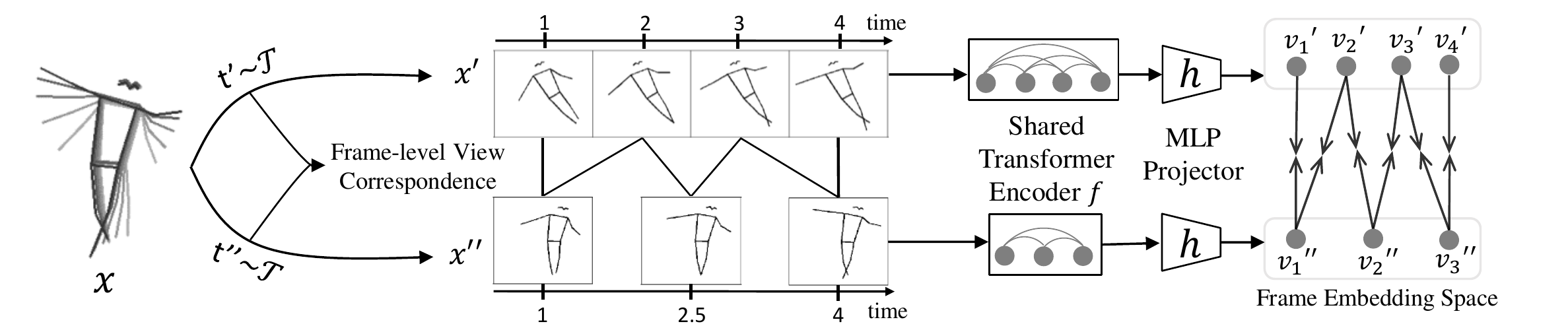}
\vspace{-5mm}
\caption{Illustration of Temporal Alignment Network (TAN), the first part of acton discovery system. Two separate augmentations are sampled from the an augmentation family $\mathcal{T}$ (\textbf{speed}, \emph{rotation and translation}) to obtain two correlated views probably of different lengths. Ground-truth correspondence between frames of the two views is obtained at the augmentation step. All frames are projected into a common  embedding space, where a \emph{frame-wise} contrastive loss is minimized.}
\vspace{-6mm}
\label{f1:spy}
\end{figure*}

\paragraph{Frame-wise Representation Learning} We observe that recurring temporal patterns, i.e.~{\em actons}, occur in dance collections across multiple dancers and genres.
Though the repeated patterns are essentially the same actions, they may vary in speed, rotation, range of motion, as well as other properties not intrinsic to the action itself.
Therefore, we first seek to learn, in a self-supervised fashion, a representation that is effective for clustering actons.

Inspired from prior works that discover {\em motifs} in the discrete domain for bioinformatics~\cite{brazma1998approaches,rigoutsos1998combinatorial} and in the continuous domain for audio~\cite{de1996unsupervised, nguyen2020zero,dannenberg2003pattern,park2007unsupervised}, 
the task of acton discovery uniquely calls for the following properties in representation learning: 

\begin{itemize}
    \item Frame-wise features. Unlike action recognition where features in the temporal dimension are usually globally pooled for an overall representation, the features in this task should be at the local frame level since the boundaries of actons have yet to be determined.
    \item Context-aware. The feature of a frame should nevertheless represent pose status in the context of a particular motion. For example, the same pose that appears in different motions should have different feature representations. 
    \item Temporal alignment. An acton can generally be performed at different speeds, so the learned feature is expected to help identify the same acton regardless of speed. The feature distance of the same frame at different motion speeds should be close to facilitate temporal alignment. 
    \item Intra-compactness and Inter-separability. A key to effective clustering is to learn discriminative features that enlarge the decision margins between clusters while reducing the variations within each cluster~\cite{liu2016large}. Intra-class compactness and inter-class separability between learned features should therefore be encouraged.
    \item Temporal continuity. The frames of an acton should be continuous in time. To facilitate acton segmentation, the noise among frames within a potential acton should be suppressed by drawing the features of adjacent frames closer together.
\end{itemize}

Formally, given a long kinematic video $\bm{x} \in \mathbb{R}^{T \times 3J}$, our goal is to learn a neural network encoder $f(\cdot)$ that extracts representation vectors $\bm{z} \in \mathbb{R}^{T \times F}$ from $\bm{x}$,
\begin{equation}
    \bm{z} = f(\bm{x}).
\end{equation}
Here, $T$ and $J$ are the number of video frames and human joints respectively and $F$ is the dimension of latent feature vectors. Note that $\bm{z}$ preserves the same temporal resolution as $\bm{x}$. 

\paragraph{Lexicon Building}
The second step for acton discovery is lexicon building, which involves two correlated tasks:
\begin{itemize}
    \item Clustering. Given a segmentation method, the training set videos are segmented into short sequences. The model is required to cluster these sequences into a lexicon. 
    \item Segmentation. Given the clustered actons, the model is required to segment an inference-time sequence into non-overlapping segments and assign an acton label to each segment. 
\end{itemize}

Earlier methods~\cite{de1996unsupervised,park2007unsupervised} mainly seek \emph{part-coverage}, which discovers isolated segments in sequences and excludes many background frames. In another problem formulation, namely \emph{full-coverage}, entire sequences are segmented and clustered into word-like units~\cite{kamper2016unsupervised,kamper2017embedded}. In this setting, these two tasks can be treated as a joint minimization problem of the sum of distances of each segment to their corresponding cluster center. In speech processing, there is a similar problem called unsupervised word discovery~\cite{nguyen2020zero} and is often addressed using probabilistic Bayesian models~\cite{kamper2017segmental,kamper2017embedded}. Recently, K-Means has been incorporated into these Bayesian models to ease the computational burden~\cite{kamper2017segmental,nguyen2020zero}. To achieve a complete tokenization, we opt for the \emph{full-coverage} formulation. 

\subsection{Temporal Alignment Network}
\label{sec: tan}

As illustrated in Figure~\ref{f1:spy}, we extend the framework of SimCLR~\cite{chen2020simple} to \textbf{frame-wise} contrastive learning, which learns representations by maximizing agreement between frames of the same semantic meaning but under different speeds.
All frame representations (for example, $\bm{z}(i)$ for the $i$-th frame) are transformed by an MLP projection head $h(\cdot)$ to a latent vector $\bm{v}(i)$ on a $F$-dim $L_2$ unit sphere,

\begin{equation}
    \bm{v}(i) = \frac{h(\bm{z}(i))}{\lVert h(\bm{z}(i)) \rVert}.
\end{equation}

As we seek here to learn frame-wise representations, we consider a frame conditioned on its temporal action context as an individual instance.
Given a minibatch of $N$ video clips $\{\bm{x}_n\}$,
denote the $i$-th frame in the video $\bm{x}_n$ as $\bm{x}_n(i)$.
We augment the sequence into two views by the same family of augmentations, $\mathcal{T}$, composed of random translation, rotation around the gravity direction, and speed change.
In order to preserve temporal smoothness, the random augmentations are applied \emph{consistently throughout a video}.
As a result, the frame $\bm{x}_n(i)$ is augmented into $\bm{x}'_n(i)$ and $\bm{x}''_n(i)$ in the two views.

We use $\bm{v}'_n(i)$ and $\bm{v}''_n(i)$ to denote the latent embeddings for the two frames.
Then, the contrastive loss function for a positive pair of frame examples is formally defined as:
\begin{equation}
\label{eq:contrastive}
\resizebox{1\hsize}{!}{%
$\mathcal{L}^{i}_{n} = -\text{log}\dfrac{\text{exp}(\bm{v}'_n(i) \cdot \bm{v}''_n(i)) /\tau)}{\text{exp}(\bm{v}'_n(i) \cdot \bm{v}''_n(i) /\tau) + \underset{x_k(j) \in D^{i}_n}{\sum} \text{exp}(\bm{v}'_n(i) \cdot \bm{v}_k(j) /\tau)},$
}
\end{equation}
where $\tau$ is the temperature parameter. It is worth noting that our contrastive instance ($\bm{z}_n(i) = f(\bm{x}_n)[i]$) differs from those in standard image or video based contrastive methods whose features are learned either from a single image ($\bm{z}_n(i) = f(\bm{x}_n[i])$)~\cite{chen2020simple} or to represent the entire video ($\bm{z}_n = f(\bm{x}_n)$)~\cite{feichtenhofer2021large}.

Due to random augmentations, some of the frames in one view may no longer have correspondences in the other view, such as when speed augmentation changes the sequence length.
We simply neglect these frames during contrastive learning. 
The notation $\mathcal{D}_n^i$ is used to denote the set of negative samples for the frame instance $\bm{x}_n(i)$. 
We symmetrize the loss in Eq.~\ref{eq:contrastive} by swapping $\bm{x}'_n(i)$ and $\bm{x}''_n(i)$ to compute $\widetilde{\mathcal{L}}^{i}_{n}$. The overall contrastive loss is formulated as:
\begin{equation}
\mathcal{L}^{TAN} = \frac{1}{2NT}\sum_{n=1}^{N}\sum_{i=1}^{T}(\mathcal{L}^{i}_{n} + \widetilde{\mathcal{L}}^{i}_{n}).
\end{equation}
\vspace{-8mm}

\paragraph{Negative Samples}
For typical instance discrimination, all instances other than the reference are considered as negatives.
However, for frame-wise representation learning, frames that are temporally close or in the same video may not be appropriate choices as negatives.
We consider the following alternatives for negative pair selection.

\begin{itemize}
    \item All frames in batch: all frames from every video in one batch other than the reference frame are considered negatives. In this case, $\mathcal{D}_n^i = \{\bm{x}_k(j)~|~ k\neq n \text{ or } j\neq i\}$. This is a straightforward extension of instance discrimination to frame-level representation learning.
    
    \item Excluding close frames: ignore negative samples within the same video clip. In this case, $\mathcal{D}_n^i = \{\bm{x}_k(j)~|~ k\neq n\}$.
    Frames within a video are usually highly correlated with one another. Ignoring them as negatives may help to learn a smooth representation over time.

\end{itemize}
For encoding the representation network $f(\cdot)$ of pose sequences in our task, we adopt the Transformer encoder network~\cite{vaswani2017attention} for its ability to learn bi-directional contextualized features. 
 
Please see Section~\ref{sec: backbone} in Supplementary Materials for more implementation details.

\subsection{Simultaneous Clustering and Segmentation}
\label{sec: cluster}

\label{sec: clustering}

Different from the iterated clustering and segmentation popular in the literature~\cite{kamper2016unsupervised,kamper2017segmental,kamper2017embedded}, we use a simplistic simultaneous clustering and segmentation algorithm relying on the features learnt from Temporal Alignment Network (TAN).

For clustering, we apply K-Means clustering on the frame-wise features over all time steps and over all the videos in a given training set. The number of actons in the lexicon is determined by the number of clusters in K-Means. The lexicon can be stored as $K$ centroid frame features. This is efficient and parallelizable~\cite{JDH17}.

For segmentation, each frame is assigned to its nearest centroid in feature space in terms of $L_2$ distance. We then segment by considering a group of consecutive frames that share the same acton cluster as a segment, as shown in Figure~\ref{teaser}. We find that our feature representation works well with this simple segmentation method. We notice that ignoring negative samples within the same video clip plays a vital role in forming smooth and clean segments.

With this segmentation technique, we segment the whole training set into acton instances and look into each acton by showing all its instances. Some examples can be found in Figure~\ref{f2:qualitative}.

\section{Evaluation Metrics} 

\subsection{Kendall’s  Tau}
To evaluate the performance of the learned representations for temporal alignment, we use the Kendall's Tau \cite{kendall1938new, dwibedi2019temporal} metric, which is a statistical measure to determine how aligned in time two sequences are in a representation space. 
Given a pair of sequences that perform the same non-repetitive action in different dance recordings with possibly different motion speed, Kendall's Tau is calculated over every pair $(u_i, u_j)$ of frames with $i<j$ in the first video. 
For all these pairs, we retrieve their nearest frames $(v_p, v_q)$ in the second video with respect to the learned representation space. A pair is called concordant if $p<q$, and otherwise it is discordant. Kendall's Tau is defined over all pairs of frames in the first video as:
\begin{align}
    \tau = \frac{\text{\# concordant pairs - \# discordant pairs}}{n(n-1)/{2}}.
\end{align}
The metric ranges from $-1$ to $1$, indicating completely reversed and aligned, respectively. 

\label{entropy}
\subsection{Normalized Mutual Information (NMI)} 
NMI is a widely used metric for estimating the quality of clustering \cite{witten2002data}. Let $Y$ be a random variable that describes the event of a testing data sample being one of the ground truth actons, while $C$ is another random variable that describes the event of a sample belonging to one of the clusters. Mutual Information measures how much knowing one of the variables reduces uncertainty about the other. For example, if knowing that a sample belongs to a cluster determines its ground truth label ($\text{H}(Y|C) = 0$, where $\text{H}(\cdot)$ denotes the entropy of a variable), then all information conveyed by $C$ is shared with $Y$; the Mutual Information is the same as the uncertainty contained in $Y$. On the other hand, if $C$ and $Y$ are independent, namely knowing that $C$ does not give any information about $Y$ ($\text{H}(Y|C) = \text{H}(Y)$), then Mutual Information is zero. Formally,
\begin{align}
    \text{NMI}(Y, C) = \frac{2(\text{H}(Y) - \text{H}(Y|C))}{\text{H}(Y) + \text{H}(C)}.
\end{align}
The NMI is normalized by the sum of entropy in $Y$ and $C$, ranging from 0 to 1. A higher NMI value indicates better clustering results.

\subsection{Language Entropy}

Language entropy is a statistical measurement on the average uncertainty (conditional entropy) of the next letter (in our case, acton), when the preceding N-1 letters are known. Specifically, let $W_N$ represent a block of contiguous actons $(w_1, w_2, ..., w_N)$ and $\mathcal{W}_N$ represent all possible progressions of $W_N$ in the given sequences. The entropy $K_{N}$ and the conditional entropy $F_{N}$ in $N$ contiguous actons are defined as:
\begin{footnotesize}
\begin{align}
    K_{N} &= -\underset{W_N \in \mathcal{W}_N}{\sum} p(W_N) \log{p(W_N)} \\
    F_{N} &= \text{K}_{N} - \text{K}_{N-1}= -\underset{W_N \in \mathcal{W}_N}{\sum} p(W_N) \log{p(w_N|W_{N-1})}.
\end{align}
\end{footnotesize}
Finally, the language entropy in~\cite{shannon1951prediction} is defined as $H = \underset{N \rightarrow \infty}{\lim} F_{N}$. 

We evaluate language entropy on testing sequences of basic choreography in the dataset, where each sequence forms only a single handcrafted labanotation token but recurs several times. Hence, the language entropy for these sequences is expected to be low. In the experiment, lower language entropy indicates better results under this setting. We observe consistent results between methods when $N \rightarrow \infty$, so we use 2-gram evaluation $F_2$ instead of $H$ in the experiments to simplify computation.

\section{Experiments}
\begin{table}[t]
\centering
\caption{Comparison of different augmentation approaches on AIST++.}
\vspace{-2mm}

\begin{tabular}{lccc}
\toprule
Aug. Method & $\tau$ $\uparrow$ & NMI $\uparrow$ & $F_2 \downarrow$ \\
\midrule
full aug. & \textbf{0.80} & \textbf{0.79} & \textbf{0.81} \\
\midrule
w/o speed aug. & 0.77 & 0.74 & 0.90 \\
w/o rotation aug. & 0.72 & 0.58 & 1.66 \\
w/o translation aug. & 0.76 & 0.76 & 0.82 \\
\bottomrule
\end{tabular}

\label{t1}
\end{table}

\begin{table}[t]
\centering
\caption{Comparison of different negative sampling on AIST++.}
\vspace{-2mm}
\begin{tabular}{lcc}
\toprule
Negative Sampling & NMI $\uparrow$ & $F_2 \downarrow$  \\
\midrule
All frames in batch & 0.35 & 2.29 \\
Excluding close frames & \textbf{0.79} & \textbf{0.81} \\
\bottomrule
\end{tabular} 

\label{t2}
\vspace{-3mm}
\end{table}

\subsection{Datasets}
\paragraph{AIST++} For our experiments, we use AIST++ Dance Motion Dataset \cite{li2021learn}\footnote{Annotations licensed by Google LLC under CC BY 4.0 license.}, which contains 3D human keypoint annotations estimated for the AIST Dance Video Database \cite{aist-dance-db}. AIST++ contains 1$,$362 sequences with 3D skeletons\footnote{After removing sequences labeled as poorly reconstructed.}, evenly distributed across 10 dance genres. For each genre, $\sim$85\% of the sequences are of \emph{basic} choreography and $\sim$15\% of them are of \emph{advanced} choreography. Basic choreographies are repetitive while advanced choreographies are longer and more complicated dances improvised by the professional dancers. Advanced videos range from 27.4 seconds (1644 frames) to 46.7 seconds (2802 frames).
\paragraph{PKU-MMD} PKU-MMD is a large-scale action detection dataset with 1$,$076 videos, which last 3 to 4 minutes and contain about 20 action instances each. There are 51 action classes, e.g. drinking, waving hand. Among other modalities, only 3D skeletons are used for our experiment. We adopted the Cross-Subject Evaluation split and report mean average precision of different actions ($\text{mAP}_\text{a}$) on the testing set with IOU threshold $\theta=0.3$.
\begin{table}[t]
\centering
\captionof{table}{Comparison with baseline methods. Accuracy refers to genre classification accuracy on AIST++ in Section~\ref{qualitative} with K=100, averaged over 6 trials. $\text{mAP}_\text{a}$ refers to action detection score on PKU-MMD.}
\vspace{-2mm}
\resizebox{0.48\textwidth}{!}{%
\begin{tabular}{lcccccc}
\toprule
&\multicolumn{3}{c}{AIST++}&\multicolumn{3}{c}{PKU-MMD}\\
\midrule
& NMI $\uparrow$ & $F_2 \downarrow$ & Acc. $\uparrow$ & NMI $\uparrow$ & $F_2 \downarrow$ & $\text{mAP}_\text{a} \uparrow$ \\
\midrule
N/A & 0.48 & 1.84 & 0.36 & 0.36 & 3.61 & 0.302 \\
TCN & 0.41 & 1.92 & 0.42 & 0.42 & 3.60 & 0.317 \\
TCC & 0.62 & 1.34 & 0.39 & 0.39 & 3.59 & 0.303 \\ \midrule
TAN & \textbf{0.79} & \textbf{0.81} & \textbf{0.46} & \textbf{0.46} & \textbf{3.50} & \textbf{0.328} \\
\bottomrule
\vspace{-8mm}
\label{t:main}
\end{tabular}
}
\end{table}

\subsection{Ablation Studies}

At the stage of \textbf{frame-wise representation learning}, we evaluate by Kendall's Tau. For each of some randomly chosen basic choreographies, we sample two videos of different tempos and crop the first complete actions that are found in common. When using 3D skeleton joint coordinates to retrieve the nearest frame by $L_2$ distance, we obtain a $\tau=0.44$, significantly lower than using trained representations from Temporal Alignment Network. We start from our final setting and ablate different augmentation components. Results are shown in Table~\ref{t1}, from which the necessity of all three augmentations can be seen.

At the stage of \textbf{lexicon building}, we evaluate by NMI and language entropy averaged across 10 genres and 15 different K values from 10 to 150 with an interval of 10. We use the choreography label of each video as the ground-truth frame label for NMI calculation. After this inference, we obtain a reorganized text corpus, on which language entropy (approximated by $F_2$) is calculated.

In Table~\ref{t1}, we observe again that all three augmentation methods are indispensable, corroborating our conclusion at the representation learning stage based on Kendall's Tau. In Table~\ref{t2}, we can see that excluding close frames from the negative samples is beneficial to TAN.

\subsection{Comparison}
We compare our model with three baseline methods, two of which are prototypical methods for self-supervised sequence representation learning. We first briefly introduce them for completeness.
\begin{figure}[h]
\centering
\includegraphics[width=0.92\linewidth]{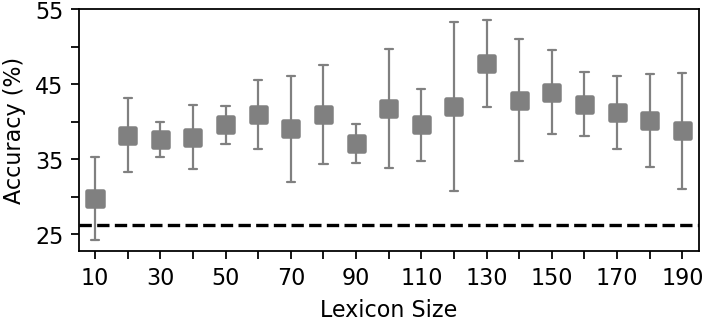}
\vspace{-3mm}  
\captionof{figure}{Mean and standard deviation over six trials of genre classification, plotted against the raw skeleton method.}
\label{fig:genre}
\end{figure}
\paragraph{Raw Skeleton Coordinates (N/A)} For this method, we use 3D skeletons directly as input for clustering. To simulate the $L_2$ distance in the frame embedding space, we calculate the $L_2$ skeleton distance after spatially normalizing the body center to the origin.

\paragraph{Time-Constrastive Networks (TCN) \cite{sermanet2018time}} 
TCN works by first sampling a certain number of anchor frames. For each of them, one positive frame within a preset time interval and one negative frame outside a larger time interval are sampled. For each motion, two augmented views are used to simulate the multi-view TC loss, where the anchor and positive frames come from the same view while negatives come from the other. After feature extraction, a triplet margin loss \cite{balntas2016learning} is used to push positive features closer to the anchor feature while pushing away negative features.

\paragraph{Temporal Cycle-Consistency (TCC) \cite{dwibedi2019temporal}} TCC is a self-supervised representation learning method based on the task of temporal alignment between videos. Compared to TCC, our method does not assume that the same event is occurring in all the videos in a dataset. We thus modify the loss by cycling between two augmented views of one sequence. For a frame $a$ in the first view, we first find its (soft) nearest neighbour in the second view $b$, then force the nearest neighbour of $b$ in the first view to be $a$ via a cross entropy loss. 

\begin{figure*}[t]
\centering
\includegraphics[width=0.8\linewidth]{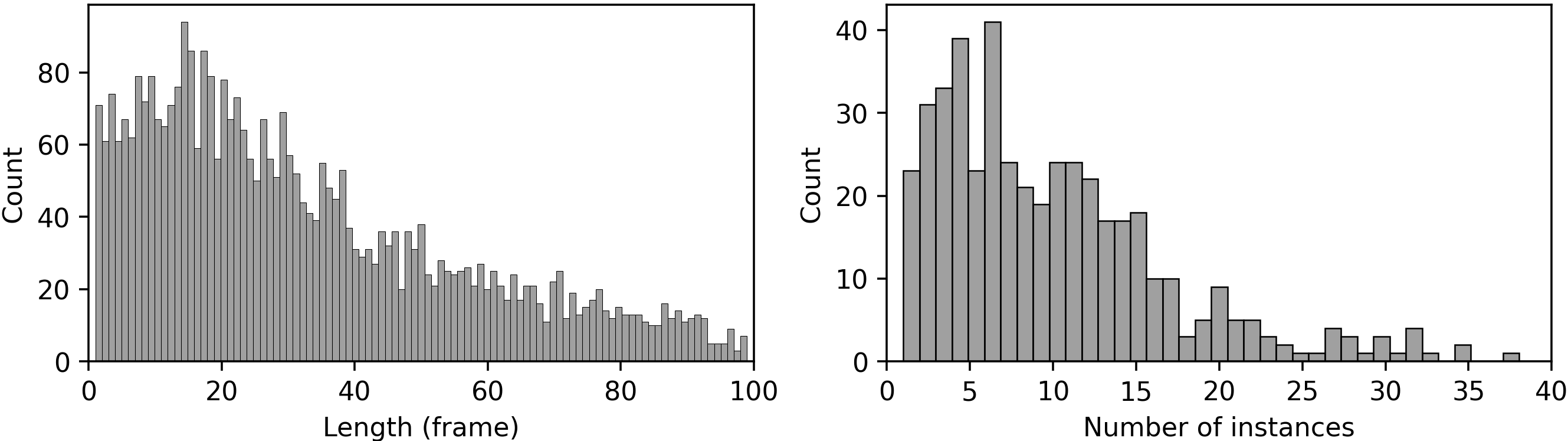}
\vspace{-1mm}
\caption{The left histogram shows the distribution of segment length across all segments. The right histogram shows the distribution of repetition number, or number of instances, for all discovered actons.}  
\label{distribution}
\end{figure*}

To compare different representation learning methods, we perform experiments by applying the same lexicon building algorithm on representations learnt from TCN, TCC, as well as using raw skeleton coordinates. For both metrics on both datasets, learnt features give better results than using raw skeleton coordinates, which justifies the use of temporal information in the representation for lexicon building. It can be seen in Table~\ref{t:main} that the lexicon built from our method shows better results on both clustering evaluation metrics and genre classification accuracy than the baselines. This might be attributed to the fact that these other methods do not utilize ground-truth frame correspondence between two augmented views.

\begin{table}[t]
\centering
\captionof{table}{Comparison with baseline method BLSTM~\cite{liu2017pku} on PKU-MMD on $\text{mAP}_\text{a}$ with varying $\theta$.}
\begin{tabular}{lcccc}
\toprule
$\theta$ & 0.1 & 0.3 & 0.5 & 0.7 \\
\midrule
BLSTM & 0.479 & 0.325 & 0.130 & 0.014 \\
TAN (ours) & \textbf{0.490} & \textbf{0.328} & \textbf{0.132} & \textbf{0.015} \\
\bottomrule
\label{t:theta}
\end{tabular}
\vspace{-5mm}
\end{table}
\subsection{Qualitative Results}

Under the real-world setting in Section~\ref{qualitative}, we visualize the acton instances (short skeleton sequences) clustered into actons in Figure~\ref{f2:qualitative}. To check the overall quality of a built lexicon, the length distribution of all segments and the distribution of the number of instances for all actons are shown in Figure~\ref{distribution}, with $K=450$. We can see that the discovered actons have a median duration of about $0.5$ second and have reasonable numbers of repetition. More qualitative results can be found in Section~\ref{app:qua} in Supplementary Material.

\begin{figure*}[h]
\centering
\includegraphics[width=\linewidth]{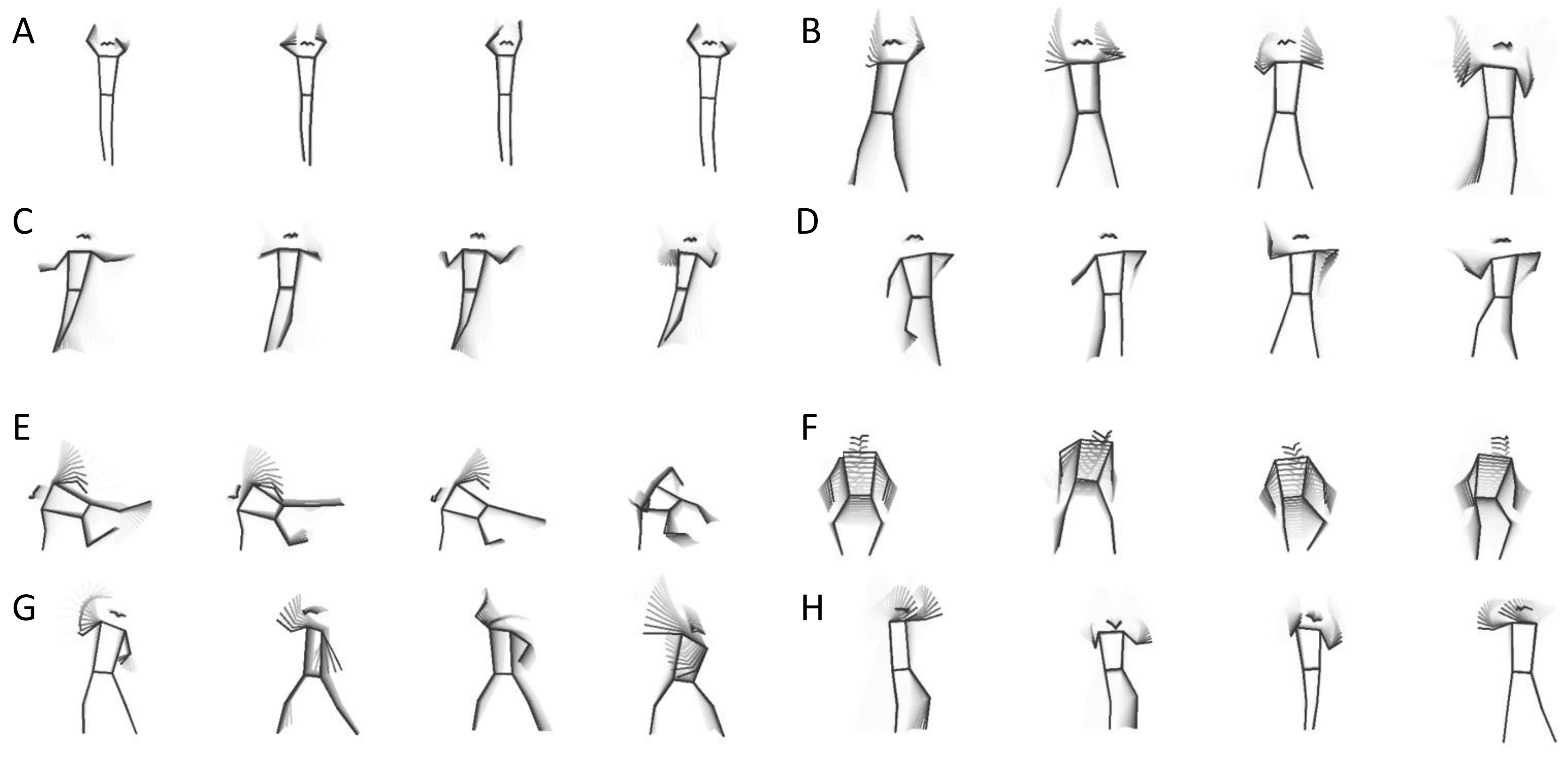}
\caption{Visualization of discovered actons from advanced videos. For each acton, four instances are shown side-to-side, with temporally earlier poses fading away. At the bottom left, an acton \textbf{G} with large intra-cluster variance is shown; to its right are four acton instances facing different directions but underpinned by the same motion.}
\vspace{-5mm}
\label{f2:qualitative}
\end{figure*}


\section{Applications}

\label{qualitative}
Acton discovery provides us with a discrete intermediate representation to support higher-level tasks on long kinematic videos. To demonstrate the utility of our tokenization for long kinematic video recognition, we conduct experiments to classify advanced choreography videos in AIST++ into ten genres and perform action detection on PKU-MMD. We also use the lexicon built on AIST++ to synthesize a \emph{never-ending} dance. Experimental details can be found in Section~\ref{detail3} in Supplementary Materials.
\subsection{Genre Classification} We use the LSTM~\cite{hochreiter1997long} text classification model with learnable and randomly-initialized word embeddings. We use a two-layer unidirectional LSTM \cite{hochreiter1997long} for sequence modeling. For tokenized input (ours), we use a learnable word embedding, while for the baseline, we use a linear layer for mapping 3D human skeleton input into a hidden space of the same dimension~\cite{li2018convolutional}. Note the significant reduction in input sequence lengths thanks to tokenization, which relieves the burden of LSTM. Test set accuracy across different lexicon sizes (K value) are presented in Figure~\ref{fig:genre}. We observe that the tokenization consistently aids classification for a reasonably large K, likely because the combinatorial structure reduces requirements on recognition system capability.
\subsection{Action Detection} After TAN pre-training and lexicon building, we learn a classifier by assigning each acton class to one of the 51 action classes defined in the PKU-MMD dataset according to maximum agreement in the training set. Using this classifier, we densely evaluate the sliding windows in different time scales across the test video, then select high-confidence local windows using the Non-Maximum Suppression (NMS) algorithm as our final action detection results. Our performance is compared with the best-performing method BLSTM~\cite{liu2017pku} in Table~\ref{t:theta}. It shows that our method is superior to BLSTM, especially with a less strict localization requirement (smaller $\theta$).
\vspace{-1mm}
\subsection{Action Composition} 
\vspace{-1mm}
In order to use the built lexicon to guide choreography (the art of designing novel motions), we present a two-stage method. First, we generate a random list of words by thresholding the $L_2$ distance between the last frame of the preceding word and the first frame of the successive word. Then we randomly instantiate each word by choosing one sequence from its cluster and splice these sequences with linear interpolation. An animated Figure~\ref{f2:qualitative} and a composed dance can be found at \url{https://bit.ly/3ctJlAy}.
\vspace{-2mm}

\section{Conclusion}
\vspace{-2mm}
In this work, we assert the importance of self-supervised acton discovery for long kinematic videos towards tokenized human dynamics representation. We then proposed a non-trivial framework composed of Temporal Alignment Network (TAN) and a Lexicon Building algorithm, and discuss some of the design choices via proposed evaluation metrics: Kendall's Tau, normalized mutual information, and language entropy. In the end, we successfully extract actons on completely unannotated action datasets, AIST++ and PKU-MMD, and reorganize them into a corpus. Among many, two directions may be especially worth pursuing in the future: discovering motifs in a corpus to build a hierarchical lexicon~\cite{shao2020finegym} and building a multi-channel lexicon by decomposing the human body into parts~\cite{hutchinson1977labanotation,li2020pastanet}.

{\small
\bibliographystyle{ieee_fullname}
\bibliography{egbib}
}
\clearpage
\newpage
\begin{onecolumn}
\begin{center}
    \section*{Supplementary Materials}
\end{center}

\section{Qualitative Results}
\label{app:qua}

More acton instances (short skeleton sequences) clustered into actons can be found in Figure~\ref{f2:qualitative'}. An animated version can be found in the supplementary demo video. 

\begin{figure}[h]
\centering
\includegraphics[width=\linewidth]{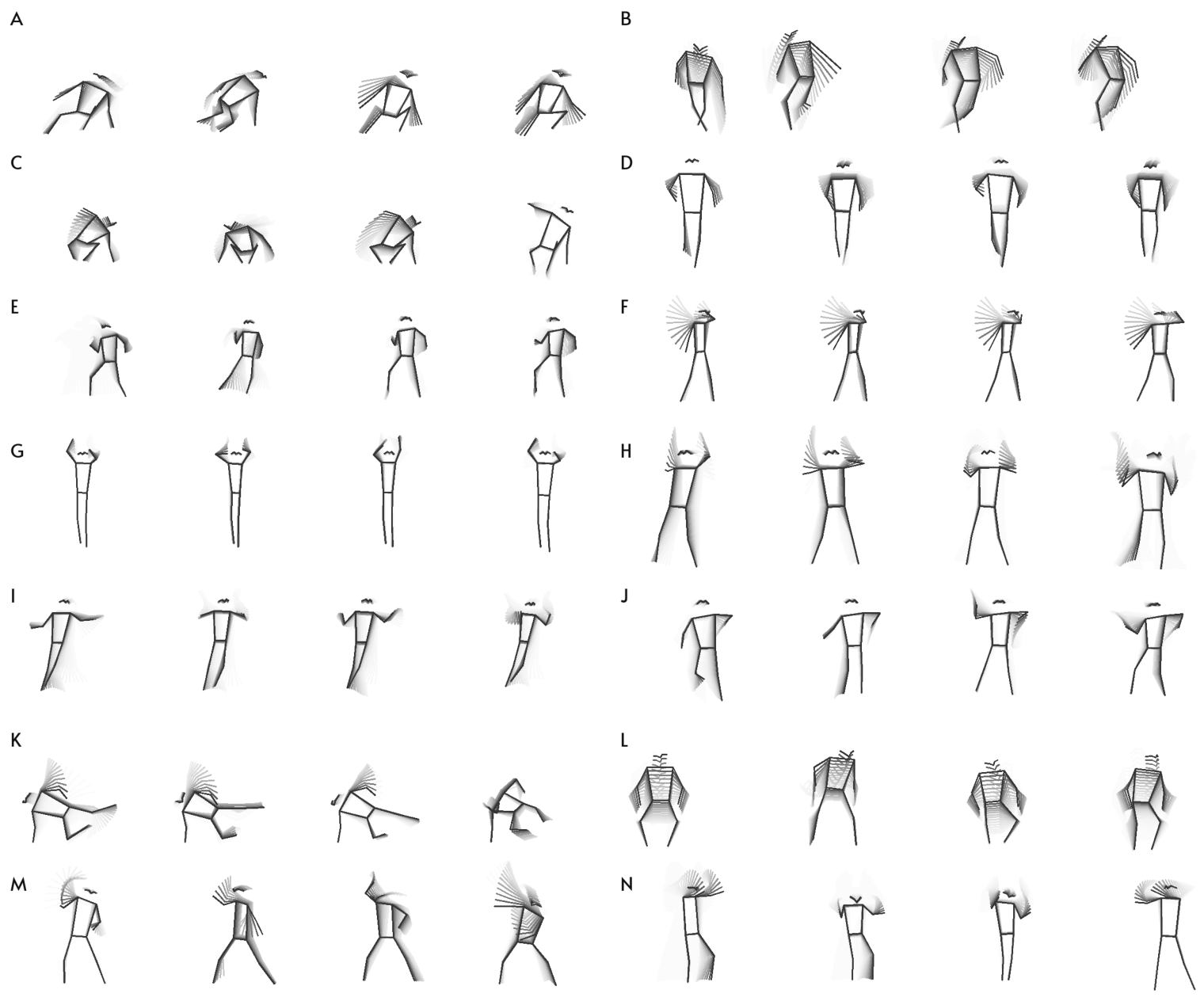}
\caption{Continuation of Figure~\ref{f2:qualitative}, visualization of discovered actons from advanced videos. For each acton, four instances are shown side-to-side, with temporally earlier poses fading away. At the bottom left, an acton \textbf{M} with large intra-cluster variance is shown; to its right are four acton instances facing different directions but underpinned by the same motion.}
\label{f2:qualitative'}
\end{figure}
\clearpage
To qualitatively assess the quality of frame embeddings produced by Temporal Alignment Network (TAN), we use t-SNE to reduce the embedding space to two dimensions and plot frames from four basic videos with different tempos. Each of these four videos is a repetition of one basic motion. From Figure~\ref{tsne}, we can learn the following:
\begin{itemize}
    \item Embeddings of frames in a video form ordered \emph{lines}, which show that our embeddings possess the quality of temporal continuity. 
    \item Despite different video lengths due to different tempos, different lines of videos have roughly the same starting, turning and ending points. This shows that our embeddings are temporally aligned.
    \item In the left part, different lines are largely overlapping, in concordance with the fact that they are the same sequence of motions. This shows that the frame embedding is capturing skeleton features. Non-overlapping parts can be attributed to the dancers rendering the same motion differently each time.
    \item Four segments in each line are overlapped, which reflects the fact that there are repeated motions for four times in each video.
    \item In the right part, we can see how lexicon building is working. It generally over-segments one \emph{motion} from a human perspective into a sequence of actons. For a conceptual example, if a labanotation describes the dance as $AAAA=(A)^4$, as it is a repetition of the same motion $A$, our system represents it as $(abcd)^4$.
    \item In the bottom part, we can see that for all four videos, the same patterns are repeated two times in the middle. It is desired that an acton level pattern recognition system can identify this and group over-segmented actons into one larger acton. This inspired us to use language entropy as a metric.
\end{itemize}
\begin{figure}[h]
\centering
\includegraphics[width=0.8\linewidth]{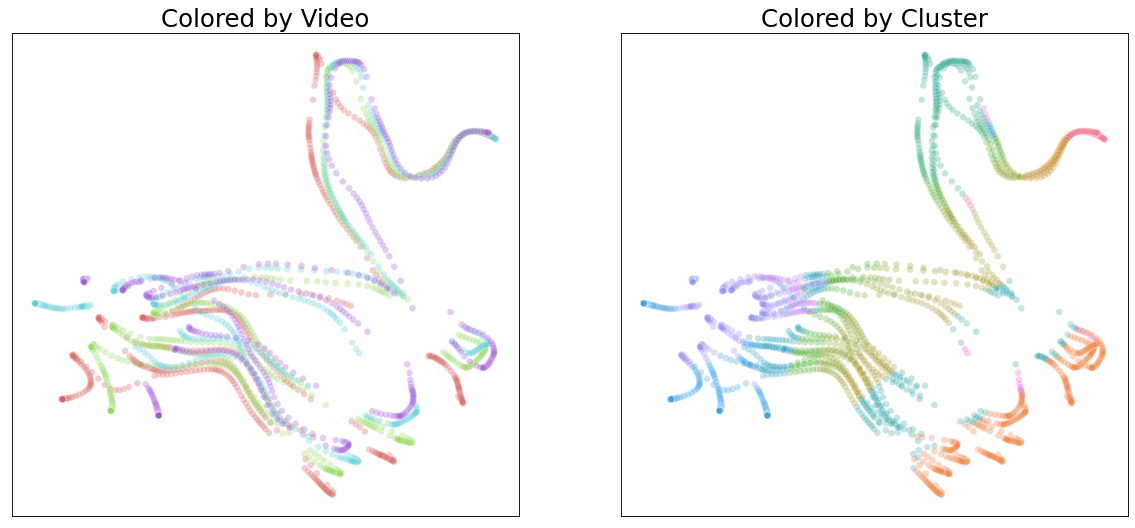}
\includegraphics[width=0.8\linewidth]{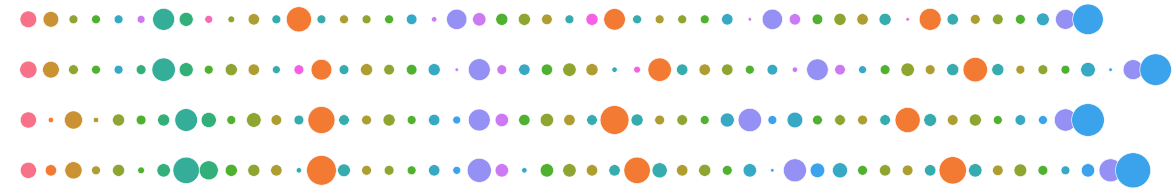}
\vspace{1mm}
\caption{t-SNE visualization of frame embeddings of videos with different tempos. On the left, frames are colored by the video they belong to; on the right, they are colored by the cluster they belong to. At the bottom is a visualization of tokenization results on the same four videos. Each line is one advanced video. Dot size corresponds to acton instance duration. Consistent colors are used for the t-SNE plot on the right and the tokenization visualization.} 
\vspace{0mm}
\label{tsne}
\end{figure}
\clearpage
Using the same setting as in Section~\ref{qualitative}, we show the reorganized corpus in Figure~\ref{color_dots}. Note that due to the specific segmentation approach used in our method, the same acton cannot appear continuously in the tokenization results.

\begin{figure}[h]
\centering
\includegraphics[width=0.7\linewidth]{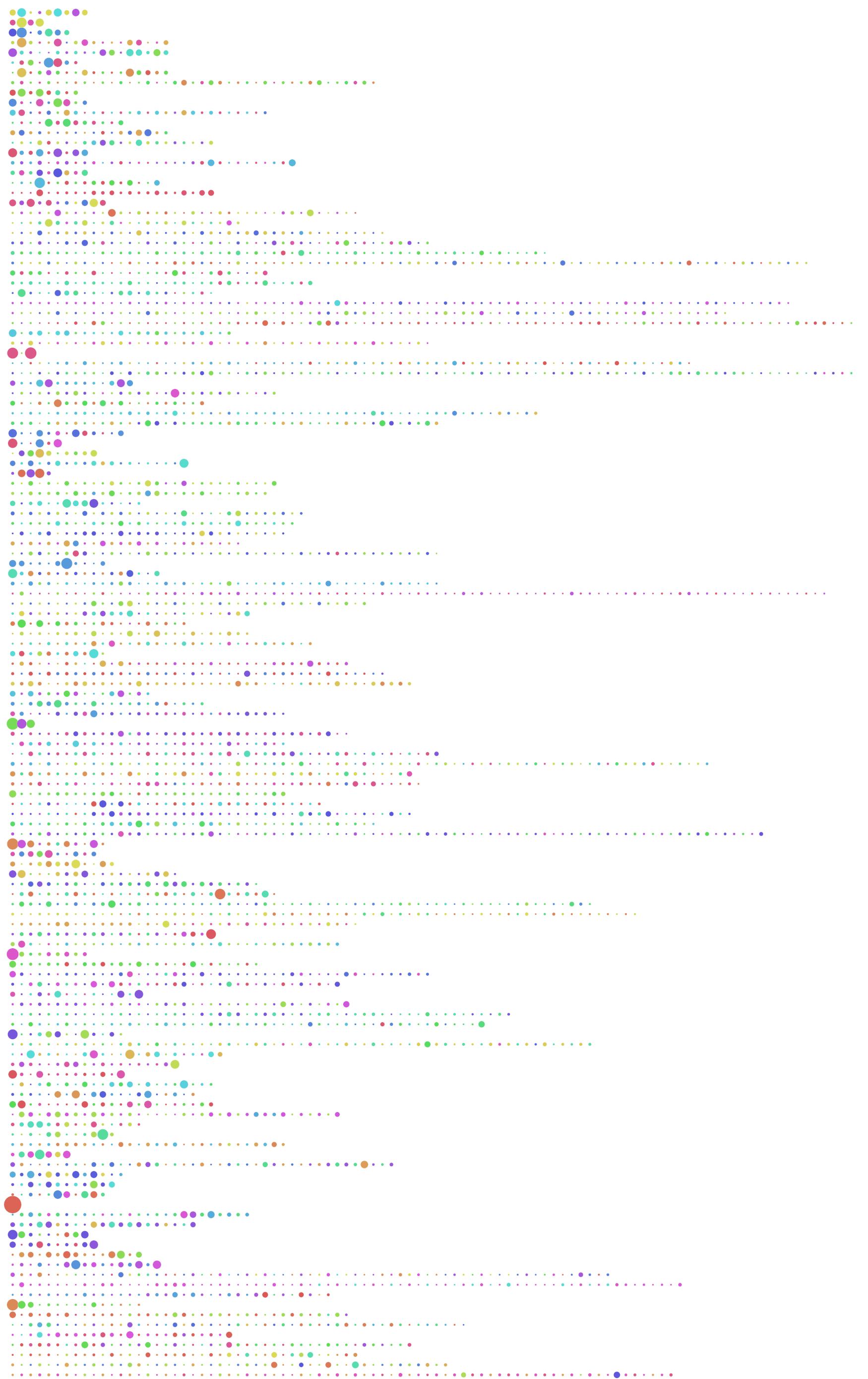}
\vspace{-3mm}
\caption{Illustration of segmentation on advanced videos in AIST++. Each line is one advanced video. Dot size corresponds to acton instance duration. Different colors represent different actons.}  
\vspace{0mm}
\label{color_dots}
\end{figure}
\clearpage
\section{Transformer Backbone Details} 
\label{sec: backbone}

For encoding the representation network $f(\cdot)$ of pose sequences in our task, we adopt the Transformer encoder network~\cite{vaswani2017attention} for its superior performance on recent vision applications~\cite{girdhar2019video,lohit2019temporal,plizzari2020spatial,carion2020end,zou2021end}.

Transformer builds layers of a representation by leveraging a global attention mechanism, and thus it is capable of learning frame-wise representations while encoding rich contextual information.
Particularly, the self-attention mechanism tends to identify the context and aggregates related information in the entire sequence to represent the current frame. 
This long-range property perfectly meets our need for context modeling. 

An encoder layer of the standard transformer architecture consists of a multi-head self-attention module and a feed-forward network (FFN)~\cite{vaswani2017attention}. Our transformer backbone network is built by stacking three encoder layers. Similar to word embedding in NLP tasks~\cite{collobert2008unified,pennington2014glove,bojanowski2017enriching}, 
a 3D skeleton sequence is simply flattened in the spatial dimension and embedded into a 512-dimensional hidden space using a two-layer MLP network. Due to the permutation invariance of the transformer architecture, standard positional encodings~\cite{vaswani2017attention} using sine and cosine functions of various frequencies are added to the input embeddings to make them sensitive to relative position in the sequence.

Note that there exists more sophisticated and potentially more effective spatial modeling architectures like GCN~\cite{yan2018spatial}, SRNet~\cite{zeng2020srnet} and Spatial Transformer~\cite{plizzari2020spatial} to model the highly structured skeleton data together with rich prior information.
However, examining them and comparing their performance is not the focus of this work.

\section{Experiment Details}
\label{detail1}

For AIST++, note that for each sequence of basic choreography, dancers are asked to dance six times with different BPMs (Beats Per Minute). Each basic choreography thus naturally forms a class. This is leveraged in part of our experiments. For each genre, we extract 20 motion sequences to form a validation set. Unless otherwise stated, representation learning is conducted on all advanced videos to prevent the system from abusing repetitive patterns in basic videos; the lexicon is built on the training split of basic videos; inference and metric calculation are done on the validation split of basic videos.  

We present the values of the hyper-parameters in Table~\ref{tab:hyperparams}. The representation learning experiments are conducted on four Nvidia V100 GPUs. 

\begin{table*}[!h]
\setlength{\tabcolsep}{0.5em}
\renewcommand{\arraystretch}{1.15}
\centering
\caption{Different hyper-parameters.}
\begin{tabular}{clc}
    \toprule
    & \textbf{Hyper-parameter} & \textbf{Value} \\ \midrule
    \multirow{10}{*}{Shared} & Batch Size & $32$ \\
    & Number of frames & $64$ \\
    & Optimizer & Adam \\
    & Peak Learning Rate & $2.5 \times 10^{-5}$\\
    & Weight Decay & $1.0 \times 10^{-6}$\\
    & Gradient Clip by Norm & $0.5$\\
    & Number of epochs & $500$ \\
    & Number of warmup epochs & $50$ \\
    & Learning rate Scheduler & \makecell{ Linear Warmup \\  Cosine Annealing } \\
    & Frames per second & $60$/$30$ (default to AIST++/PKU-MMD) \\ \midrule
    \multirow{4}{*}{Specific to TCN} & Anchor Number & $16$ \\
    & Positive Window Size & $2$ \\
    & Negative Multiplier & $4$ \\
    & Triplet Loss Margin & $2$ \\
    \bottomrule
\end{tabular}
\label{tab:hyperparams}
\end{table*}

\section{Augmentation Details}
\label{detail2}
In Table~\ref{tab:aug}, we show the hyper-parameters of different data augmentations. For translation augmentation and rotation augmentation, the probability of different values used are uniform, while for speed augmentation, we first uniformly sample a number between $1$ and the highest speed, then give a $50\%$ chance of using its reciprocal. 

\begin{table}[!h]
\setlength{\tabcolsep}{0.5em}
\renewcommand{\arraystretch}{1.15}
\centering
\caption{Hyper-parameters of the standard augmentations.}
\begin{tabular}{lc}
    \toprule
    \textbf{Hyper-parameter} & \textbf{Value} \\ \midrule
    Translation range & $\pm 0.2 \text{m}$ \\
    Rotation range & $\pm 18 ^\circ $ \\
    Speed range & $\frac{1}{2} \sim 2 $ \\
    \bottomrule
\end{tabular}
\label{tab:aug}
\end{table}

\section{Implementation Details of Applications}
\label{detail3}
\paragraph{Genre Classification} For our genre classification method, we create 
a new split of advanced choreographies with a larger test set in relation to the development set.
Representation learning and lexicon building is carried out on the development set. After lexicon building, we do nearest neighbour inference on the test set to obtain tokenizations of the test videos.  Hyperparameters can be found in Table~\ref{tab:recog}.

\paragraph{Model Selection} For both variants, we further split the AIST++ development split into a training split and a validation split by a ratio of $4:1$. Sizes of the sets can be found in Table~\ref{tab:recog}. We use the best-performing model on this validation split for testing. Results on the test split are reported. 

\begin{table}[!h]
\setlength{\tabcolsep}{0.5em}
\renewcommand{\arraystretch}{1.15}
\centering
\caption{Hyper-parameters used for long kinematic video recognition experiment}
\vspace{2mm}
\begin{tabular}{lc}
    \toprule
    \textbf{Hyper-parameter} & \textbf{Value} \\ \midrule
    Batch Size & $50$ \\
    Optimizer & Adam \\
    Peak Learning Rate & $1.0 \times 10^{-4}$\\
    Weight Decay & $1.0 \times 10^{-4}$\\
    Gradient Clip by Norm & $0.1$\\
    Number of epochs & $200$ \\
    Hidden Space Dimension & $256$ \\
    Train set size & $109$ \\
    Validation set size & $28$ \\
    Test set size & $61$ \\
    Class number & $10$ (default to AIST++) \\
    \bottomrule
\end{tabular}
\label{tab:recog}
\end{table}

\section{More on Proposed Metrics}

Here, we provide a proof that the conditional entropy sequence $\{F_N\}$ defined in Section~\ref{entropy} converges. Intuitively, the more characters that are conditioned on, the less uncertainty there is in predicting the next character. 

\begin{prop}
The sequence $F_N$ converges. 
\end{prop}

\begin{proof}
Since $F_N$ is a conditional entropy sequence, we know $F_N \geq 0$. We only need to prove $F_{N+1} \geq F_N$ for any $N \in \mathbb{N}^+$.

\begin{align*}
    F_N - F_{N+1} &= - \sum_{\mathcal{W}_N} p(W_N) \log p(w_N | W_{N-1}) + \sum_{\mathcal{W}_{N+1}} p(W_{N+1}) \log p(w_{N+1} | W_N) \\
    & = \sum_{\mathcal{W}_{N-1}} \left( \underset{w_N, w_{N+1}}{\sum} p(W_{N+1}) \log p(w_{N+1} | W_{N-1}, w_N) - \sum_{w_N} p(W_N) \log p(w_N | W_{N-1}) \right) \\
    & \geq \sum_{\mathcal{W}_{N-1}} \left( \underset{w_N, w_{N+1}}{\sum} p(W_{N+1}) \log p(w_{N+1} | W_{N-1}) - \sum_{w_N} p(W_N) \log p(w_N | W_{N-1}) \right) \\
    & = \sum_{\mathcal{W}_{N-1}} \left( \underset{w_{N+1}}{\sum} p(W_{N-1}, w_{N+1}) \log p(w_{N+1} | W_{N-1}) - \sum_{w_N} p(W_N) \log p(w_N | W_{N-1}) \right) \\
    & = 0.
\end{align*}
\end{proof}
The value it converges to is then defined as language entropy $H$ by \cite{shannon1951prediction}.

We measure the correlation strengths among the three proposed metrics, namely, Kendall's Tau, NMI, and language entropy, on $1000$ runs of experiments differing only in augmentation parameters with Pearson's $r$, Spearman's $\rho$, and Kendall's $\tau$ \cite{scikit-learn}. 
From Table~\ref{correlation}, we can interpret the results as indicating that NMI and language entropy have strong correlation. 

\begin{table}[!h]
\centering
\caption{Correlation among proposed metrics}
\vspace{-2mm}
\begin{tabular}{lccc}

\toprule
& $|r| \uparrow$ & $\rho \uparrow$ & $\tau \uparrow$ \\
\midrule
Kendall's Tau \& NMI & 0.45 & 0.45 & 0.31 \\
Kendall's Tau \& language entropy & 0.35 & 0.36 & 0.24 \\
NMI \& language entropy & 0.82 & 0.81 & 0.62 \\ 
\bottomrule
\label{correlation}
\end{tabular}
\end{table}

\section{Limitations}
\label{limitation}
One obvious limitation of our work is that it has not been applied to RGB image input yet, which significantly restricts usable datasets. Though there exists no difficulty in adapting the architecture design to extend our pipeline, we plan to replace the skeleton encoder with an image encoder in later works.

\end{onecolumn}
\end{document}